\documentclass[conference]{IEEEtran}
\usepackage{times}

\IEEEoverridecommandlockouts

\usepackage[numbers]{natbib}
\usepackage{multicol}
\usepackage[bookmarks=true]{hyperref}
\usepackage{booktabs} 
\usepackage{graphicx} 
\usepackage{gensymb}  
\usepackage{algorithm}
\usepackage[noend]{algpseudocode}
\usepackage{amsmath}
\usepackage{amsthm}
\usepackage{amssymb}
\usepackage{amsfonts}
\usepackage{subcaption}
\usepackage{color}
\usepackage{array}
\usepackage{mathtools}
\newcolumntype{x}[1]{>{\centering\let\newline\\\arraybackslash\hspace{0pt}}p{#1}}  

\usepackage{tikz}
\usetikzlibrary{backgrounds,matrix,automata,positioning,calc}

\DeclareMathOperator*{\argmin}{arg\,min}
\DeclarePairedDelimiter\norm{\lVert}{\rVert}

\theoremstyle{definition}
\newtheorem{definition}{Definition}[section]
\newtheorem{proposition}{Proposition}
\newtheorem{theorem}{Theorem}
\newtheorem{problem}{Problem}
\newtheorem{remark}{Remark}

\begin{document}

\title{Toward Verifiable Real-Time Obstacle Motion Prediction for Dynamic Collision Avoidance}

\author{Vince Kurtz and Hai Lin

\thanks{The partial support of the National Science Foundation (Grant No.
ECCS-1253488, IIS-1724070, CNS-1830335) and of the Army Research
Laboratory (Grant No. W911NF- 17-1-0072) is gratefully acknowledged.}
\thanks{Vince Kurtz and Hai Lin are with the Department of Electrical Engineering,
University of Notre Dame, Notre Dame, IN, 46556 USA. \texttt{\url{vkurtz@nd.edu}},\texttt{
\url{hlin1@nd.edu}}}
}

\maketitle

\begin{abstract}
Next generation Unmanned Aerial Vehicles (UAVs) must reliably avoid moving obstacles. Existing dynamic collision avoidance methods are effective where obstacle trajectories are linear or known, but such restrictions are not accurate to many real-world UAV applications. We propose an efficient method of predicting an obstacle's motion based only on recent observations, via online training of an LSTM neural network. Given such predictions, we define a Nonlinear Probabilistic Velocity Obstacle (NPVO), which can be used select a velocity that is collision free with a given probability. We take a step towards formal verification of our approach, using statistical model checking to approximate the probability that our system will mispredict an obstacle's motion. Given such a probability, we prove upper bounds on the probability of collision in multi-agent and reciprocal collision avoidance scenarios. Furthermore, we demonstrate in simulation that our method avoids collisions where state-of-the-art methods fail.
\end{abstract}

\IEEEpeerreviewmaketitle

\section{Introduction and Related Work}
To enable future UAV applications like disaster response, infrastructure inspection, and package delivery, autonomous UAVs must avoid collisions with moving obstacles. Such obstacles might include hobbyist drones, aircraft with limited maneuverability, and other UAVs using the same collision avoidance policy. The motion of these obstacles may not be known a priori, and they will not reliably continue at their current velocity. However there is often some structure underlying each obstacle's movement such that we might predict future motion by observing past behavior. A simple example of such motion is shown in Figure \ref{fig:motion}. 

We propose a novel algorithm to predict the motion of moving obstacles via online training of an LSTM Recurrent Neural Network (RNN) \cite{hochreiter1997long}, using dropout to obtain uncertainty estimates over these predictions \cite{gal2016dropout,gal2016theoretically}. To our knowledge, this is the first work proposing an online obstacle motion prediction system for collision avoidance without a priori environmental knowledge or extensive offline training.

Fergusen et al \cite{ferguson2008detection} show that environmental structure can be used to make predictions about obstacle motion for autonomous vehicles. In less structured environments, Hug et al \cite{pedestrian2018} show that LSTM neural networks can predict pedestrian paths from extensive training data. They note that LSTM is well-suited to time-varying patterns with both short and long-term dependencies. We extend this work by learning patterns of obstacle motion online rather than from a large training set. Furthermore, we use dropout sampling to approximate a probability distribution such that the variance of this distribution can be used as an uncertainty estimate \cite{kahn2017uncertainty}. 

Given probabilistic predictions of obstacle movement, we propose an uncertainty-aware multi-agent dynamic collision avoidance algorithm based on Nonlinear Probabilistic Velocity Obstacles (NPVO), a novel extension of existing velocity obstacle notions. These include Probabilistic Velocity Obstacles, which provide an uncertainty-aware policy in static environments \cite{fulgenzi2007dynamic}, and Nonlinear Velocity Obstacles \cite{shiller2001motion}, which can guarantee collision avoidance for obstacles moving along known trajectories. Our NPVO, which considers the future behavior of an obstacle in a probabilistic sense, is a generalization of these notions.

The state-of-the-art ``Optimal Reciprocal Collision Avoidance" (ORCA) algorithm uses reciprocal velocity obstacles to control multiple agents in unstructured environments \cite{van2011reciprocal}. This approach is popular due to its ease of implementation and guarantees that a collision-free trajectory will be found for $\tau$ time, if one is available. However, ORCA assumes that all obstacles in the workspace are either static or operating according to the same policy. Violations of this assumption can lead to catastrophic behavior, as shown in Figure \ref{fig:orca_fail}. We demonstrate in simulation that our NPVO approach is able to avoid such obstacles, and prove bounds on our algorithm's safety in reciprocal and multi-agent scenarios.

\begin{figure}
    \centering
    \begin{subfigure}{0.18\textwidth}
        \centering
        \includegraphics[width=\textwidth]{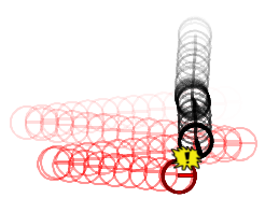}
        \caption{ORCA \cite{van2011reciprocal}}
        \label{fig:orca_fail}
    \end{subfigure}
    \begin{subfigure}{0.18\textwidth}
        \centering
        \includegraphics[width=\textwidth]{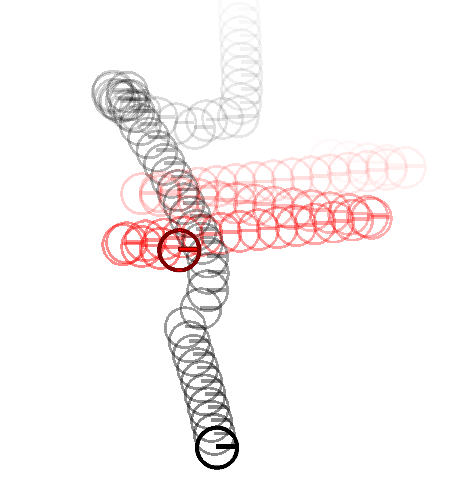}
        \caption{Ours}
        \label{fig:our_success}
    \end{subfigure}
    \caption{An obstacle (red) moves back and forth while slowly drifting downwards. Existing methods fail to avoid such an obstacle, while our system accurately predicts its motion and avoids a collision. }
    \label{fig:motion}
\end{figure}

Finally, we take an important step towards rigorous verification of our framework. Existing results for formal verification of systems based on techniques like LSTM are highly limited \cite{carlini2017towards,huang2017safety,katz2017reluplex}, but formal guarantees are of vital importance for safety-critical applications like collision avoidance. We propose a novel statistical model checking formulation to approximate the probability that an obstacle will remain within certain bounds. Along the way, we demonstrate that predictions generated by our algorithm are robust to perception uncertainty in the form of additive Gaussian noise. 

Given the probability that obstacles will remain within the bounds of our predictions (which we approximate with statistical model checking), we prove several important safety properties of our algorithms. First, we prove a bound on the probability of collision between one agent and one (moving) obstacle. Next, we prove a tighter bound on the probability of collision for two agents reciprocally avoiding collisions. We use these results to prove bounds on the probability of collision for one agent avoiding $N$ obstacles, as well as on the probability of any collision occuring between $N$ agents using our proposed approach for reciprocal collision avoidance.

The remainder of this paper is organized as follows: Section \ref{sec:problem_formulation} formulates the multi-agent collision avoidance problem.
, while Section \ref{sec:nn} provides some necessary background and notation regarding RNNs and dropout. 
Sections \ref{sec:pred} and \ref{sec:npvo} describe our obstacle motion prediction and NPVO collision avoidance algorithms. Section \ref{sec:stat_mc} introduces a novel method for approximating the probability that our prediction algorithm will make a faulty prediction. Given such a probability, we prove several important safety properties of our collision avoidance algorithms in Section \ref{sec:proof}. We present simulation results in Section \ref{sec:simulation} and conclude the paper with Section \ref{sec:conclusion}. 

\section{Problem Formulation}\label{sec:problem_formulation}

Consider a robotic agent $A$ with single-integrator dynamics and perfect actuation. That is, the position of agent $A$ at instant $k$, $\mathbf{p}^A_k$, is governed by
\begin{equation}
    \mathbf{p}^A_{k+1} = \mathbf{p}^A_k + \mathbf{v}_k\Delta t
\end{equation}
where $\mathbf{v}$ is a velocity (control input), and $\Delta t$ is a sampling period. Assume there exists some desired velocity $\mathbf{v}_{des}$, provided a priori or by a higher-level planner.

Suppose that agent $A$ operates in a workspace with $N$ dynamic obstacles $B_1, B_2, \dots , B_N$. Each $B_i$ moves according to an unknown and possibly time-varying policy:
\begin{equation}\label{eqn:obs}
    \mathbf{p}^{B_i}_{k+1} = f_i(\mathbf{p}^{B_i}_k,k,\mathbf{u}_k^{B_i})
\end{equation}
where $\mathbf{p}^{B_i}_k$ denotes the position of obstacle $B_i$ at instant $k$, and $\mathbf{u}_k^{B_i}$ denotes some unknown  control input. We assume perfect observability of the position of all $B_i$ for the last $n$ timesteps. That is, at instant $n$,
$$
    \mathcal{O} = \{\mathbf{p}^{B_i}_0,\mathbf{p}^{B_i}_1,...,\mathbf{p}^{B_i}_n\}
$$
is known to agent $A$.

Finally, assume that there exists some safe distance $r_s$ such that for any obstacle $B_i$, if
$$
    \norm{\mathbf{p}^A_k - \mathbf{p}^{B_i}_k}_{\mathcal{L}_2} \geq r_s,
$$
then $A$ and will not collide with $B_i$ at instant $k$. 

The multi-agent dynamic collision avoidance problem can then be stated as follows:

\begin{problem}{\textit{Multi-Agent Dynamic Collision Avoidance}}\label{prob:dca}

    Given a preferred velocity $\mathbf{v}_{des}$ and observations $\mathcal{O} = \{\mathbf{p}^{B_i}_0,\mathbf{p}^{B_i}_1,...,\mathbf{p}^{B_i}_n\}_{i=1}^N$, find a safe velocity $\mathbf{v}_{safe}$ that minimizes the probability of collision for the next $m$ timesteps, which is given by
    \begin{align}\label{eqn:p_coll}
        \mathbb{P}(\norm{\mathbf{p}_{n+k}^A-\mathbf{p}_{n+k}^{B_i}} < r_s \text{ for any } k \in [1,m], i \in [1,N]).
    \end{align}
\end{problem}

\section{Neural Networks}\label{sec:nn}

\subsection{Recurrent Neural Networks for Timeseries Analysis}

Given a sequence of timeseries data $\{\mathbf{x}_1, \mathbf{x}_2, ..., \mathbf{x}_n\}$, the output $\mathbf{y}_i$ of a simple Recurrent Neural Network (RNN) is given by 
$$\mathbf{h}_i = \sigma(\mathbf{W}_h\mathbf{x}_i + \mathbf{U}_h\mathbf{h}_{i-1} + \mathbf{b}_h)$$
$$\mathbf{y}_i = \sigma(\mathbf{W}_y\mathbf{h}_i + \mathbf{b}_y)$$
where $\mathbf{h}_i$ is the ``hidden state'' at step $i$, $\sigma(\cdot)$ is some nonlinear function (typically a rectified linear unit or sigmoid), and $\mathcal{W} = \{\mathbf{W}_h, \mathbf{U}_h, \mathbf{W}_y, \mathbf{b}_h, \mathbf{b}_y\}$ are the weights (and biases) of the network.

Simple RNNs are powerful tools for processing timeseries data, but often struggle on problems with long-term temporal dependencies. The LSTM (Long Short Term Memory) RNN structure \cite{hochreiter1997long} addresses this shortcoming with the addition of cell state $\mathbf{C}_i$ as well as several gates, which regulate how information is passed to the cell state:
$$\mathbf{f}_i = \sigma(\mathbf{W}_f\mathbf{x}_i + \mathbf{U}_f\mathbf{h}_{i-1} + \mathbf{b}_f)$$
$$\mathbf{i}_i = \sigma(\mathbf{W}_\mathbf{i}\mathbf{x}_i + \mathbf{U}_\mathbf{i}\mathbf{h}_{i-1} + \mathbf{b}_\mathbf{i}) $$
$$\mathbf{o}_i = \sigma(\mathbf{W}_o\mathbf{x}_i + \mathbf{U}_o\mathbf{h}_{i-1} + \mathbf{b}_o)$$
$$\mathbf{\Tilde{C}}_i = tanh(\mathbf{W}_C\mathbf{x}_i + \mathbf{U}_C\mathbf{h}_{i-1} + \mathbf{b}_C)$$
$$\mathbf{C}_i = \mathbf{f}_i \odot \mathbf{C}_{i-1} + \mathbf{i}_i \odot \mathbf{\Tilde{C}_{i}}$$
$$\mathbf{y}_i = \mathbf{h}_i = \mathbf{o}_i \odot tanh(\mathbf{C}_i).$$
Here $\odot$ denotes the Hadamard product, $tanh(\cdot)$ the hyperbolic tangent function, the nonlinear activation $\sigma(\cdot)$ is usually the sigmoid function, and $\mathcal{W} = \{ \mathbf{W}_f, \mathbf{U}_f, \mathbf{b}_f, \mathbf{W_i}, \mathbf{U_i}, \mathbf{b_i}, \mathbf{W}_o, \mathbf{U}_o, \mathbf{b}_o, \mathbf{W}_C, \mathbf{U}_C, \mathbf{b}_C\}$ are the weights and biases of the network. 

Throughout this paper, we use the shorthand notation $$\mathbf{y} = \mathcal{NN}(\mathbf{x} ; \mathcal{W})$$ to indicate the output $\mathbf{y} = \mathbf{y}_n$ from passing input $\mathbf{x} = \{ \mathbf{x}_i \}_{i=1}^n$ through a general network $\mathcal{NN}$ parameterized by weights $\mathcal{W}$.

\subsection{Dropout for Uncertainty Quantification}\label{sec:dropout}

Dropout was originally proposed as a technique for regularizing the output of neural networks. The basic idea is to set individual weights from $\mathcal{W}$ to zero at random. Applying dropout during the training process ensures that no individual weight is relied upon too much, which can cause overfitting. 

For a neural network with input $\mathbf{x} \in \mathbb{R}^n$, for example, we can apply dropout by considering the new input
$$
\mathbf{\hat{x}} = \mathbf{z} \odot \mathbf{x}
$$
where $\mathbf{z} \in \mathbb{R}^n$ and the $i^{th}$ element of $\mathbf{z}$ is sampled from $Bernoulli(p)$.

Gal et al. showed in \cite{gal2016dropout} that in addition to providing regularization, dropout can be used to approximate a Bayesian Neural Network, in which weights and outputs are treated as probability distributions. Futhermore, they show in \cite{gal2016theoretically} that this result is applicable to recurrent networks as well, with the caveat that the same dropout mask should be applied at each step, and to both input and the hidden state $\mathbf{h}$. That is, for a recurrent network (including an LSTM network) with input data $\mathbf{x}_i$ and hidden state $\mathbf{h}_i$, we use a new input and a new hidden state defined by 
$$
\mathbf{\hat{x}}_i = \mathbf{z}_x \odot \mathbf{x}_i
$$
$$
\mathbf{\hat{h}}_i = \mathbf{z}_h \odot \mathbf{x}_i
$$
where $\mathbf{z}_h$ and $\mathbf{z}_x$ are vectors of $0$s and $1$s as above. This is different from most use cases of dropout in recurrent networks for regularization, where the dropout mask change at each timestep:
$$
\mathbf{\hat{x}}_i = \mathbf{z}_{x_i} \odot \mathbf{x}_i
$$
$$
\mathbf{\hat{h}}_i = \mathbf{z}_{h_i} \odot \mathbf{x}_i.
$$

Using dropout to approximate a Bayesian Neural Network allows us to consider the model uncertainty inherent in any predictions. This approach has been used with great effectiveness for uncertainty-aware reinforcement learning \cite{kahn2017uncertainty}. To our knowledge, however, this paper presents the first application of dropout for uncertainty estimation in obstacle motion prediction for collision avoidance. 

\section{Online Prediction of Obstacle Motion}\label{sec:pred}

To solve Problem \ref{prob:dca}, we first propose an algorithm to predict the motion of moving obstacles based on past observations. Our algorithm provides predictions as \textit{probability distributions} over future obstacle positions, which will later allow us to estimate and minimize the probability of collision (Equation (\ref{eqn:p_coll})). In this section, we consider the motion of a single obstacle $B_i$. For simplicity, we denote this obstacle's position at instant $k$ as $\mathbf{p}_k^{B_i} = \mathbf{p}_k$. To predict the motion of several obstacles, separate instances of the algorithms described in this section can be used. 

Given a set of observations $\mathcal{O} = \{ \mathbf{p}_0,\mathbf{p}_1, \dots, \mathbf{p}_n \}$, we approximate equation \ref{eqn:obs} by
\begin{equation}
    \mathbf{p}^{B_i}_{n+1} = g(\mathcal{O})
\end{equation}
and approximate $g(\mathcal{O})$ online using an LSTM network \cite{hochreiter1997long}.

First, we calculate a training set of known inputs and outputs. Observations are transformed from positions to changes in position, so that the input to the network is invariant with respect to shifts in the workspace. That is, we calculate $\hat{\mathcal{O}} = \{ \Delta\mathbf{p}_1, \Delta\mathbf{p}_2, ... , \Delta\mathbf{p}_n \}$
where $\Delta\mathbf{p}_k = \mathbf{p}_{k}-\mathbf{p}_{k-1}$.
We then consider the first $k$ changes in position 
$$
\{\Delta\mathbf{p}_1,\Delta\mathbf{p}_2,\dots,\Delta\mathbf{p}_k\}
$$
as an input datapoint and the subsequent $m$ changes in position
$$
\{\Delta\mathbf{p}_k,\Delta\mathbf{p}_{k+1},\dots,\Delta\mathbf{p}_{k+m}\}
$$
to be the corresponding output. To model perception uncertainty, we perturb each $\Delta\mathbf{p}$ with zero-mean Gaussian noise with variance $\sigma^2$. Changing $\sigma^2$ will allow us to verify the robustness of our system to perception uncertainty (see Section \ref{sec:stat_mc}). The complete training dataset at timestep $n$, $d_n = (\mathbf{X}, \mathbf{Y})$, is then specified as follows:
$$
\mathbf{X} = \{ \mathbf{x}_k \}_{k=1}^{n-m}, ~~ \mathbf{Y} = \{ \mathbf{y}_k \}_{k=1}^{n-m}
$$
where
$$
\mathbf{x}_k = \{ \Delta\mathbf{p}_i + \epsilon_i \}_{i=1}^k, ~~ \epsilon_i \sim G(0, \sigma^2)
$$
$$
\mathbf{y}_k = \{ \Delta\mathbf{p}_i + \epsilon_i \}_{i=k+1}^{k+m}, ~~ \epsilon_i \sim G(0, \sigma^2).
$$

For a given input $\mathbf{x}_k$, we denote the output of the neural network parameterized by weights $\mathcal{W}$ as
$$
\hat{\mathbf{y}}_k = \mathcal{NN}(\mathbf{x}_k ; \mathcal{W}).
$$
We then find weights to minimize minimize the cost function 
$$
C_k(\mathcal{W}) = L_\delta[\mathbf{y}_k - \mathbf{\hat{y}}_k]
$$
using a fixed number of iterations ($N_{iter}$) of the stochastic gradient descent algorithm Adam \cite{kingma2014adam}. $L_\delta[\cdot]$ indicates the Huber norm \cite{huber1964robust} with parameter $\delta$. This process is summarized in Algorithm \ref{alg:train}.

\begin{algorithm}
    \caption{TrainNetworkOnline}\label{alg:train}
    \begin{algorithmic}[1]
        \Procedure{TrainNetwork}{$\{ \Delta\mathbf{p}_1, ..., \Delta\mathbf{p}_{n} \}$}
        \For{$k = [1, n-m]$}        
            \State $\mathbf{x}_k = \{ \Delta\mathbf{p}_i + \epsilon_i\}_{i=1}^{k}, ~~~ \epsilon_i \sim G(0,\sigma^2)$
            \vspace{0.5em} 
            \State $\mathbf{y}_k = \{ \Delta\mathbf{p}_{i} + \epsilon_i \}_{i=k+1}^{k+m}, ~~~ \epsilon_i \sim G(0,\sigma^2)$
            \vspace{0.5em} 
            \State $\hat{\mathbf{y}}_k = \mathcal{NN}(\mathbf{x_k}; \mathcal{W})$
            \State $C_k(\mathcal{W}) = L_{\delta}[\mathbf{y}_k - \hat{\mathbf{y}}_k]$
        \EndFor
        \State $\mathcal{W}^* = \argmin \sum_{k}{C_k(\mathcal{W})}$
        \vspace{0.5em} 
        
        \Return $\mathcal{NN}(\cdot;\mathcal{W}^*)$
        \EndProcedure
    \end{algorithmic}
\end{algorithm}

The second part of our online learning approach is prediction, which is outlined in Algorithm \ref{alg:ellipse}. We consider the whole history of observed position changes, perturbed by perception noise, as input to a neural network with weights $\mathcal{W}$:
$$
\mathbf{x} = \{\Delta\mathbf{p}_i + \epsilon_i \}_{i=1}^n ~~ \epsilon_i \sim G(0, \sigma^2).
$$
Applying dropout as per \cite{gal2016theoretically}, we can treat each output of the network as a prediction of changes in position for the next $m$ timesteps:
$$
\{ \Delta\mathbf{\hat{p}}_{n+1}, \Delta\mathbf{\hat{p}}_{n+2}, \dots , \Delta\mathbf{\hat{p}}_{n+m} \} = \mathcal{NN}(\mathbf{x}).
$$

With repeated application of dropout, we can construct a set of predictions
$$
\hat{Y} = \{ \Delta\mathbf{\hat{p}}_{n+1}^i,\Delta\mathbf{\hat{p}}_{n+2}^i, \dots , \Delta\mathbf{\hat{p}}_{n+m}^i \}_{i=1}^{N_{s}}.
$$
Note that $\hat{Y}$ contains $N_s$ predictions for each timestep, $\{ \Delta\mathbf{\hat{p}}_{n+k}^i \}_{i=1}^{N_s}$. We consider these predictions to be samples from an underlying multivariate Gaussian distribution $G(\mu_k, \Sigma_k)$. This distribution represents the probability that a certain motion, $\Delta\mathbf{p}_k$, will be taken by the obstacle at instant $k$. We note that our approach could be extended in a straightforward way to other (possibly multimodal) distributions, but we focus here on a unimodal Gaussian for ease of illustration. 

Given sample predictions $\{ \Delta\mathbf{\hat{p}}_{n+k}^i \}_{i=1}^{N_s}$, we calculate the maximum likelihood estimates of $\mu_k$ and $\Sigma_k$:
$$
\hat{\mu}_k = \frac{1}{N_s}\sum_{i=1}^{N_s} \Delta\mathbf{\hat{p}}_{n+k}^i
$$
$$
\hat{\Sigma}_k = \frac{1}{N_s}\sum_{i=1}^{N_s} (\Delta\mathbf{\hat{p}}_{n+k}^i - \hat{\mu}_k)(\Delta\mathbf{\hat{p}}_{n+k}^i - \hat{\mu}_k)^T.
$$
This allows us to estimate the position of the obstacle at instant $n+k$ as
$$
\hat{\mathbf{p}}_{n+k} = \sum_{i=n+1}^{n+k}\mu_i + \mathbf{p}_n.
$$

We can then write a new set of distributions that reflect the position of the obstacle in the future. Writing the position of the obstacle at time $n+k$ as a random variable $\mathbf{P}_{n+k}$, we have
$$
\mathbf{P}_{n+k} \sim G(\hat{\mathbf{p}}_{n+k}, \Sigma_k).
$$
Then, given some threshold $\gamma$, we construct ellipsoids $e_k$ in position space such that
$$
\mathbb{P}(\mathbf{P}_{n+k} \in e_k) \geq \gamma ~~ \forall k \in [1,m].
$$
These ellipsoids $\{e_k\}_{k=1}^m$ will be used later to construct an NPVO (see Section \ref{sec:npvo}, Algorithm \ref{alg:whole_system}).

\begin{algorithm}
    \caption{PredictObstacleMotion}\label{alg:ellipse}
    \begin{algorithmic}[1]
        \Procedure{PredictMotion}{$ \{ \Delta\mathbf{p}_i\}_{i=1}^{n}, \mathcal{NN}, \gamma$}
        \State $\mathbf{x} = \{ \Delta\mathbf{p}_i + \epsilon_i\}_{i=1}^{n}, ~~~ \epsilon_i \sim G(0,\sigma)$
        \vspace{0.3em}
        \State $\hat{Y} = \{\}$ 
        \vspace{0.3em}
        \State $ellipoids = \{\}$
        \vspace{0.1em}
        
        \For{$j=[1,N_{s}]$ }
            \vspace{0.3em} 
            \State $\{ \Delta\mathbf{\hat{p}}_{i}\}_{i=n+1}^{n+m} = \mathcal{NN}(\textbf{x})$
            \Comment{with dropout}
            \vspace{0.3em} 
            \State $\hat{Y} \gets \{ \Delta\mathbf{\hat{p}}_{i}\}_{i=n+1}^{n+m}$
        \EndFor
        
        \For{$k=[1,m]$}
            \vspace{0.3em} 
            \State $\mathbf{\mu}_k, \Sigma_k = MLE(\{ \Delta\mathbf{p}_{n+k}^i \}_{i=1}^{N_s})$
            \vspace{0.3em} 
            \State $\hat{\mathbf{p}}_{k} = \mathbf{p}_n + \sum_{i=n+1}^{k}\mathbf{\mu}_i $
            \vspace{0.3em} 
            \State $e_k = \{ \mathbf{p} \mid \mathbb{P}(\mathbf{p} \in e_k) \geq \gamma \} $
            \vspace{0.3em} 
            \State $ellipsoids \gets e_k$
        \EndFor
        
        \Return ellipsoids 
        \EndProcedure        
    \end{algorithmic}

\end{algorithm}

\begin{figure}
    \centering
    \begin{tikzpicture}
        \node (input) at (0,0) {$\{\Delta\mathbf{p}_i\}_{i=1}^{n}$};
        \node (output) at (6.5,0.7) {$\{\Delta\mathbf{\hat{p}}_{i} \}_{i=n+1}^{n+m}$};
        \node[draw] at (3.0,0.7) (pred) {Prediction Net};
        \node[fill=white,draw] at (3.0,-0.7) (train) {Training Net};
        
        \draw[thick,->] (input) -- ([xshift=-1mm] pred.west);
        \draw[thick,->] (input) -- ([xshift=-1mm] train.west);
        \draw[thick,->] (pred.east) -- (output);
        
        \draw[thick,dashed,->] ([xshift=-5.5mm] train.north) node[left,yshift=4.5mm,xshift=-1mm] {$\mathcal{W}$} to[out=120,in=-120] ([xshift=-5mm] pred.south);
       
        \begin{scope}[on background layer] 
            \draw[thick,->] (train.east) -- (4.5,-0.7) -- (4.5,-1.35) -- (2.7,-1.35) -- (2.7,-1.2) -- (3.7,-0);
        \end{scope}
    \end{tikzpicture}
    \caption{Two copies of the LSTM network are used in parallel for online prediction.}
    \label{fig:online_threading}
\end{figure}
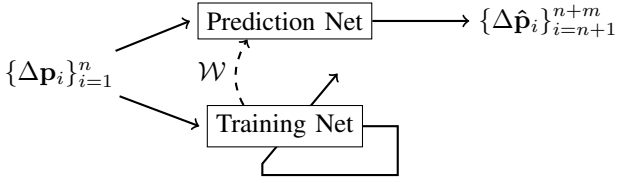

To achieve real-time online training and prediction, we use the multithreading approach shown in Figure \ref{fig:online_threading}. The reason for this is that Algorithm \ref{alg:train} is not guaranteed to terminate within $\Delta t$ time, in which case it is important to be able to still make predictions. Two identical copies of the network, one for training and one for prediction, are created. Algorithm \ref{alg:train} determines weights of the training network. These weights are then copied to the prediction network and used in Algorithm \ref{alg:ellipse}. The only difference between the two networks is the dropout mask used: the training network always uses different masks at each timestep for regularization, while the prediction network uses the same mask at each timestep to approximate a Bayesian network \cite{gal2016theoretically}. 

\section{Nonlinear Probabilistic Velocity Obstacles}\label{sec:npvo}

In this section, we shown how the predictions generated in Section \ref{sec:pred} can be used for collision avoidance. In doing so, we propose a new velocity obstacle concept, the NPVO. The traditional velocity obstacle is defined as follows:

\begin{definition}{\textit{Velocity Obstacle. \cite{van2011reciprocal}}}
    \label{def:VO}
    The Velocity Obstacle for agent $A$ induced by agent $B$ for time window $\tau$ is the set of all velocities of $A$ that will result in a collision between $A$ and $B$ within $\tau$ time:
        $$
        VO_{A \mid B}^{\tau} = \{ \mathbf{v} \mid \exists t \in [0,\tau] ~ :: ~  \mathbf{v}t \in D(\mathbf{p}_B - \mathbf{p}_A, r_{s}) \}
        $$
    where $D(\mathbf{p}, r)$ denotes a disk of radius $r$ centered at position $\mathbf{p}$, $\mathbf{p}_A$ ($\mathbf{p}_B$) is the position of agent $A$ ($B$), and $r_{s}$ is the minimum safe distance between agents.
\end{definition}

Our NPVO is a straightforward extension of this notion to discrete time, under the assumption that we can obtain some probabilistic estimate of the future position of a moving obstacle. The basic definition, assuming a single obstacle in the workspace, is given below:

\begin{definition}{\textit{Nonlinear Probabilistic Velocity Obstacle.}}
    \label{def:NPVO}
    
    Assume the position of obstacle $B$ at timestep $k$ is estimated by the random variable $\mathbf{P}^B_k \sim \mathcal{F}_k$. The NPVO for agent $A$ induced by obstacle $B$ for $m$ timesteps with probability $\gamma$ is then given by
    $$
    NPVO_{A\mid B}^{m, \gamma} = 
    $$
    $$
    \{ \mathbf{v} \mid \exists k \in [0,m] ~ :: ~ \mathbb{P}( \mathbf{v}k \in  D(\mathbf{p}^A - \mathbf{P}^B_k, r_{s})) > \gamma \}
    $$
    where $D(\mathbf{p}, r)$ denotes a disk of radius $r$ centered at position $\mathbf{p}$, $\mathbf{p}^A$ is the current position of agent $A$, and $r_{s}$ is the minimum safe distance between agents.
\end{definition}

Given a nonlinear probabilistic velocity obstacle and a desired velocity, we can specify a safe velocity via the constrained optimization problem
$$
\argmin( \norm{\mathbf{v}_{des} - \mathbf{v}_{safe}}_{\mathcal{L}_2}^2),
$$
$$
\mathbf{v}_{safe} \notin NPVO_{A\mid B}^{\gamma,m},
$$
where $\mathbf{v}_{des}$ is a desired velocity. $\mathbf{v}_{safe}$ is guaranteed to avoid a collision for the next $m$ timesteps with a given probability, as long as there exists a dynamically feasible $\mathbf{v} \notin NPVO_{A\mid B}^{\gamma,m}$. 

The NPVO concept can be easily applied to achieve collision avoidance with the prediction system outlined in Section \ref{sec:pred}, as shown in Algorithm \ref{alg:whole_system}. 

\begin{algorithm}
    \caption{NPVO Collision Avoidance}\label{alg:whole_system}
    \begin{algorithmic}[1]
        \Procedure{CollAvoid}{$\mathbf{p}_n$, $\{ \Delta\mathbf{p}_i\}_{i=1}^n$, $\gamma$, $\mathbf{v}_{des}$, $goal$}
        \While{$goal$ not reached}
            
            \State $\mathcal{NN} = \textsc{TrainNetwork}(\{ \Delta\mathbf{p}_1, ..., \Delta\mathbf{p}_{n} \})$
            
            \State $\{e_k\}_{k=1}^m = \textsc{PredictMotion}(\{ \Delta\mathbf{p}_i \}_{i=1}^n, \mathcal{NN},\gamma)$
            
            \State $NPVO_{A\mid B}^{\gamma,m} = $
            
            $\{ \mathbf{v} \mid \exists k \in [1,m] ~ :: ~ \mathbf{v}k + r \in e_k \forall r \text{ s.t.} |r| \leq r_{safe} \}$
            
            \State $\mathbf{v}_{safe} =\argmin( \lVert\mathbf{v}_{des} - \mathbf{v}_{safe}\rVert^2)$
            
            ~~~~~~~~~~ s.t. $\mathbf{v}_{safe} \notin NPVO_{A\mid B}^{\gamma,m}$\label{alg:whole_system:vsafe}
            
            \State apply $\mathbf{v}_{safe}$
        \EndWhile
        \EndProcedure
    \end{algorithmic}
\end{algorithm}

The notion of an NPVO can be easily extended to the multi-agent case, allowing us to use our prediction paradigm for multi-agent collision avoidance. 

\begin{definition}{\textit{Multi-Agent NPVO}}
    \label{def:multiagent_NPVO}
    
    Assume the positions of $N$ obstacles $\{B_i\}_{i=1}^N$ at timestep $k$ can be estimated by the random variables $\mathbf{P}_{B_i} \sim \mathcal{F}_k^i$. The NPVO for agent $A$ induced by $\{B_i\}_{i=1}^N$ for $m$ timesteps with probability $\gamma$ is then given by
    $$
        NPVO_{A\mid \{B_1, B_2, ... , B_N\}}^{m, \gamma} =
    $$
    \begin{multline*}
        \{ \mathbf{v} \mid \exists k \in [0,m] ~ :: ~
        \mathbb{P}( \mathbf{v}k \in  D(\mathbf{p}^A - \mathbf{P}^{B_i}_k, r_{s})) \geq \gamma \\
        \text{ for any } i \in [1,N]\}
    \end{multline*}
    where $D(\mathbf{p}, r)$ denotes a disk of radius $r$ centered at position $\mathbf{p}$, $\mathbf{p}^A$ is current the position of agent $A$, and $r_{s}$ is the minimum safe distance between $A$ and any agent $B_i$.
\end{definition}

The corresponding procedure for multi-agent collision avoidance using an NPVO is outlined in Algorithm \ref{alg:multiagent_system}.

\begin{algorithm}
    \caption{Multi-agent NPVO Collision Avoidance}\label{alg:multiagent_system}
    \begin{algorithmic}[1]
        \Procedure{MultiCA}{$\{\mathbf{p}_n^j\}_{j=1}^N$, $\{\{ \Delta\mathbf{p}^j_i \}_{i=1}^n\}_{j=1}^N$, $\gamma$, $\mathbf{v}_{des}$, $goal$}
        \While{$goal$ not reached}
            \For{$j \in [1,N]$}
                \State $\mathcal{NN}^j = \textsc{TrainNetwork}(\{ \Delta\mathbf{p}_1^j, ..., \Delta\mathbf{p}_{n}^j \})$
                
                \State $\{e^j_k\}_{k=1}^m = $
                
                \hfill$\textsc{PredictMotion}(\{ \Delta\mathbf{p}_i^j\}_{i=1}^n, \mathcal{NN}^j,\gamma)$
            \EndFor
            
            \State $NPVO_{A\mid B_1, B_2,...,B_N}^{\gamma,m} =$
            
            \hfill$\{ \mathbf{v} \mid \exists k \in [1,m] ~ :: ~ \mathbf{v}k \in e^j_k \text{ for any } j \in [1..N] \}$
            
            \State $\mathbf{v}_{safe} =\argmin( \lVert\mathbf{v}_{des} - \mathbf{v}_{safe}\rVert^2)$
            
            \hfill s.t. $\mathbf{v}_{safe} \notin NPVO_{A\mid B_1,B_2,...,B_N}^{\gamma,m}$
            
            \State apply $\mathbf{v}_{safe}$
        \EndWhile
        \EndProcedure
    \end{algorithmic}
\end{algorithm}
\section{Statistical Model Checking}\label{sec:stat_mc}
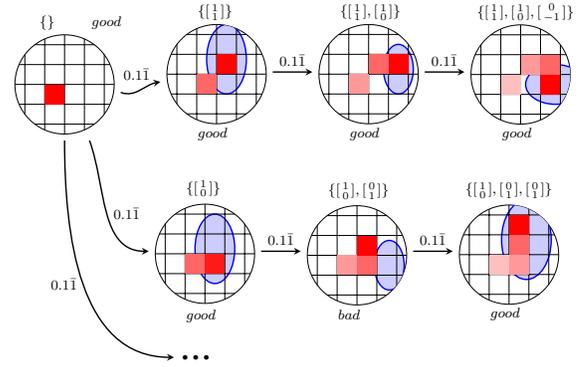
\begin{figure}
    \centering
    \begin{tikzpicture}[scale=0.66, ->, >=stealth, node distance=0.8cm ]
        \useasboundingbox (-2,-6) rectangle (11,2);
        
        \pgflowlevelsynccm
        
        \node[matrix] (s1) 
        {
            {\node[text width=1.0cm] at (0.0,1.25) {\footnotesize $\{\}$};}
            {\node[text width=0.3cm] at (0.7,1.25) {\footnotesize $good$};}
            \clip[draw] circle (1cm);
            \draw[step=0.4cm] (-1,-1) grid (1,1); 
            \fill[red!99!white] (-0.4,-0.4) rectangle (0.0,0.0);\\
        };
        
        \node[matrix, right=of s1] (s2) 
        {
            {\node[text width=1.0cm] at (-0.5,1.25) {\footnotesize $\{[\begin{smallmatrix}1\\1\end{smallmatrix}]\}$};}
            {\node[text width=1.0cm] at (-0.5,-1.2) {\footnotesize $good$};}
            \clip[draw] circle (1cm);
            {\draw[blue,thick,fill=blue!20] (0.2,0.3) ellipse (0.4cm and 0.7cm);}
            \draw[step=0.4cm] (-1,-1) grid (1,1); 
            \fill[red!60!white] (-0.4,-0.4) rectangle (0.0,0.0);
            \fill[red!99!white] (-0.0,-0.0) rectangle (0.4,0.4);\\
        };
        \node[matrix, right=of s2] (s3) 
        {
            {\node[text width=1.0cm] at (-0.7,1.25) {\footnotesize $\{[\begin{smallmatrix}1\\1\end{smallmatrix}],[\begin{smallmatrix}1\\0\end{smallmatrix}]\}$};}
            {\node[text width=1.0cm] at (-0.5,-1.2) {\footnotesize $good$};}
            \clip[draw] circle (1cm);
            {\draw[blue,thick,fill=blue!20] (0.6,0.1) ellipse (0.3cm and 0.5cm);}
            \draw[step=0.4cm] (-1,-1) grid (1,1); 
            \fill[red!40!white] (-0.4,-0.4) rectangle (0.0,0.0);
            \fill[red!60!white] (-0.0,-0.0) rectangle (0.4,0.4);
            \fill[red!99!white] (0.4,0.0) rectangle (0.8,0.4);\\
        };
        \node[matrix, right=of s3] (s4) 
        {
            {\node[text width=1.5cm] at (-1,1.25) {\footnotesize $\{[\begin{smallmatrix}1\\1\end{smallmatrix}],[\begin{smallmatrix}1\\0\end{smallmatrix}],[\begin{smallmatrix}0\\-1\end{smallmatrix}]\}$};}
            {\node[text width=1.5cm] at (-0.5,-1.2) {\footnotesize $good$};}
            \clip[draw] circle (1cm);
            {\draw[blue,thick,fill=blue!20] (0.7,-0.2) ellipse (0.6cm and 0.4cm);}
            \draw[step=0.4cm] (-1,-1) grid (1,1); 
            \fill[red!25!white] (-0.4,-0.4) rectangle (0.0,0.0);
            \fill[red!40!white] (-0.0,-0.0) rectangle (0.4,0.4);
            \fill[red!60!white] (0.4,0.0) rectangle (0.8,0.4);
            \fill[red!99!white] (0.4,-0.4) rectangle (0.8,0.0);\\
        };
       
        \node[matrix, below right=of s1] (s5) 
        {
            {\node[text width=0.5cm] at (-0.5,1.6) {\footnotesize $\{[\begin{smallmatrix}1\\0\end{smallmatrix}]\}$};}
            {\node[text width=0.5cm] at (-0.5,-1.0) {\footnotesize $good$};}
            \clip[draw] circle (1cm);
            {\draw[blue,thick,fill=blue!20] (0.2,0.1) ellipse (0.4cm and 0.7cm);}
            \draw[step=0.4cm] (-1,-1) grid (1,1); 
            \fill[red!60!white] (-0.4,-0.4) rectangle (0.0,0.0);
            \fill[red!90!white] (-0.0,-0.4) rectangle (0.4,0.0);\\
        };
        \node[matrix, right=of s5] (s6) 
        {
            {\node[text width=1cm] at (-0.7,1.3) {\footnotesize $\{[\begin{smallmatrix}1\\0\end{smallmatrix}],[\begin{smallmatrix}0\\1\end{smallmatrix}]\}$};}
            {\node[text width=0.5cm] at (-0.5,-1.2) {\footnotesize $bad$};}
            \clip[draw] circle (1cm);
            {\draw[blue,thick,fill=blue!20] (0.65,-0.2) ellipse (0.3cm and 0.5cm);}
            \draw[step=0.4cm] (-1,-1) grid (1,1); 
            \fill[red!40!white] (-0.4,-0.4) rectangle (0.0,0.0);
            \fill[red!60!white] (-0.0,-0.4) rectangle (0.4,0.0);
            \fill[red!99!white] (-0.0,-0.0) rectangle (0.4,0.4);\\
        };
        \node[matrix, right=of s6] (s7) 
        {
            {\node[text width=1.5cm] at (-1,1.3) {\footnotesize $\{[\begin{smallmatrix}1\\0\end{smallmatrix}],[\begin{smallmatrix}0\\1\end{smallmatrix}],[\begin{smallmatrix}0\\1\end{smallmatrix}]\}$};}
            {\node[text width=1.5cm] at (-0.5,-1.2) {\footnotesize $good$};}
            \clip[draw] circle (1cm);
            {\draw[blue,thick,fill=blue!20] (0.35,0.3) ellipse (0.5cm and 0.8cm);}
            \draw[step=0.4cm] (-1,-1) grid (1,1); 
            \fill[red!25!white] (-0.4,-0.4) rectangle (0.0,0.0);
            \fill[red!40!white] (0.0,-0.4) rectangle (0.4,0.0);
            \fill[red!60!white] (0.0,0.0) rectangle (0.4,0.4);
            \fill[red!99!white] (0.0,0.4) rectangle (0.4,0.8);\\
        };
        
        \node[text width=1cm, below=of s5, yshift=0.5cm] (dotdotdot) {\Huge{...}};
       
        \path[every node/.style={font=\sffamily\footnotesize}]
          (s1) edge[out=-20, in=-170, thick] node[above,xshift=-0.07cm, yshift=0.07cm] {${0.1\bar1}$} (s2)
          (s2) edge[thick] node[above] {${0.1\bar1}$} (s3)
          (s3) edge[thick] node[above] {${0.1\bar1}$} (s4)
          
          (s1) edge[out=-70, in=180, thick] node[right] {${0.1\bar1}$} (s5)
          (s5) edge[thick] node[above] {${0.1\bar1}$} (s6)
          (s6) edge[thick] node[above] {${0.1\bar1}$} (s7)
          
          (s1) edge[out=-90, in=180, thick] node[left] {${0.1\bar1}$} (dotdotdot);
    \end{tikzpicture}
    \caption{An example Markov Chain used to verify our prediction system. Red shaded squares indicate the path of a (hypothetical) obstacle. The state $s = \{\Delta\mathbf{p}_1, ..., \Delta\mathbf{p}_n\}$ is indicated above each node. Transitions are drawn uniformly from a 9-cell grid (not all are shown). Predictions are indicated by blue ellipses: a state is $good$ if the latest position is within the last prediction.}
    \label{fig:markov_chain}
\end{figure}

Rigorously verifying that systems based on neural networks satisfy certain specifications, even probabilistically, is a difficult and open problem \cite{carlini2017towards,huang2017safety,katz2017reluplex}.  In safety-critical applications like dynamic collision avoidance, however, it is important to provide some guarantees of safety and performance. In this section, we demonstrate a novel means of providing approximate probabilistic guarantees for our prediction system.

While it might be tempting to reason about theoretical guarantees directly from the probability distribution estimated from dropout, it is well established that dropout provides only a lower bound on model uncertainty \cite{gal2016dropout,kahn2017uncertainty}. Therefore, we use \textit{statistical model checking} to obtain more rigorous results. Specifically, we are interested in estimating the probability that the an obstacle $B$ will remain within prediction ellipsoids $e^B_k$ throughout in the next $m$ timesteps:
\begin{equation}\label{eqn:model_checking_spec}
    \mathbb{P}(\mathbf{p}^B_{n+k} \in e^B_k ~ \forall ~ k \in [1,m]).
\end{equation}
As shown in Section \ref{sec:proof}, estimating this probability will allow us to prove important properties of Algorithms \ref{alg:whole_system} and \ref{alg:multiagent_system}. 

The general problem of statistical model checking is defined as follows:

\begin{definition}{Statistical Model Checking}
    \label{def:statistical_model_checking}
    
    Given a model $\mathcal{M}$ and a property $\phi$, determine if $\mathbb{P}(\mathcal{M} \vDash \phi) \geq \theta$, where $\theta$ is a desired performance threshold. 
\end{definition}

Typically, $\mathcal{M}$ is a stochastic model such as a Markov Chain (MC), and $\phi$ is encoded in a logic such as Probabilistic Computation Tree Logic (PCTL) \cite{hansson1994logic}. Assuming that $\phi$ can be determined on finite executions of $\mathcal{M}$, we define
$
B_i \sim Bernoulli(p)
$
such that $b_i = 1$ if the $i^{th}$ execution of $\mathcal{M}$ satisfies $\phi$. Model checking is then reduced to choosing between two hypothesis: 
$$H_0 : p \geq \theta + \delta$$
$$H_1 : p < \theta - \delta$$
where $\delta > 0$ denotes some indifference region. In this work, we use the Sequential Probability Ratio Test (SPRT) \cite{wald1945sequential,legay2010statistical}, to choose between $H_0$ and $H_1$ with a given strength $(\alpha, \beta)$\footnote{$\alpha$ and $\beta$ denote the maximum probabilities of Type I and Type II error respectively} using minimal samples $m$. 


To verify the performance of our prediction system, we first compute the MC abstraction $$\mathcal{M} = (\mathcal{S}, \mathit{Init}, T, AP, L)$$ where
\begin{itemize}
    \item $\mathcal{S} \ni s = \{ \Delta\mathbf{p}_1, ..., \Delta\mathbf{p}_n \}$, $\Delta\mathbf{p}_i \in \mathcal{P}$, where $\mathcal{P}$ is a set of possible position changes
    \item $\mathit{Init} = \{ \}$
    \item $T(s,s') = \mathbb{P}(\Delta\mathbf{p}_{n+1} \mid s) = \begin{cases}
        \frac{1}{|\mathcal{P}|} & \Delta\mathbf{p}_{n+1} \in \mathcal{P} \\
        0 & \text{else}
    \end{cases}$
    \item $AP = \{ \mathit{bad}, \mathit{good} \}$
    \item $L(s) = \begin{cases}
        good & \mathbf{p}_{n-m+k} \in e_k ~ \forall k \in [1,m] \\
        bad & \text{else}
    \end{cases}$
\end{itemize}
and $\mathbf{p}_j = \sum_{i=1}^j \Delta\mathbf{p}_i$.
An example of this abstraction is shown in Figure \ref{fig:markov_chain}. 

\begin{remark}\label{remark}
Uniform random transitions $T(s,s')$ form a \textit{worst case scenario} for our prediction system: since the prediction system described in Section \ref{sec:pred} relies on structure underlying the timeseries data, its performance given uniform random transitions should give a lower bound on the performance of the system. 
\end{remark}

With this in mind, given a sufficiently fine gridding $\mathcal{P}$, we can estimate the value of (\ref{eqn:model_checking_spec}) by checking the PCTL property 
\begin{equation}
    \phi^* = \mathbb{P}_{\geq \theta}( \square \mathit{good} )
\end{equation}
for various values of $\theta$. The highest value of $\theta$ such that $\phi^*$ is satisfied is an upper bound (recalling Remark (\ref{remark})) of the probability of mispredicting an obstacle's motion. 

Unfortunately, existing statistical model checking methods do not handle unbounded properties like $\phi^*$ well \cite{legay2010statistical}. Instead, we verify the related PCTL property
\begin{equation}
    \label{eq:specification}
    \phi = \mathbb{P}_{\geq \theta}( \square^{\leq N} \mathit{good} ).
\end{equation}

For sufficiently large $N$, the largest $\theta$ such that the model $\mathcal{M}$ satisfies $\phi$ serves as an \textit{approximation} of the probability of mispredicting an obstacle's motion (Equation (\ref{eqn:model_checking_spec})).

\subsection{Example}

As an example, we define a 3x3 grid
$$
\mathcal{P} =  \{\Delta\mathbf{p}\} = \{ \begin{bmatrix}
\Delta x \in \{-1,0,1\} \\
\Delta y \in \{-1,0,1\}
\end{bmatrix}  \}.
$$
For all trials, we use the SPRT to determine the satisfaction of Equation \ref{eq:specification} with $N=20$, $\alpha = 0.1$, $\beta = 0.1$, and $\delta = 0.05$. The results for select values of $\theta$ are shown in Table \ref{tab:stat_mc}.

\begin{table}
    \centering
    \begin{tabular}{@{}x{2cm}x{1.1cm}x{1.1cm}x{1.1cm}x{1.1cm}@{}}
        \toprule
             & $\theta=0.9$ & $\theta=0.85$ & $\theta=0.8$ & $\theta=0.75$ \\
        \midrule
            $\sigma^2=0$ & UNSAT & UNSAT & SAT & SAT \\
            $\sigma^2=0.001$ & UNSAT & SAT & SAT & SAT \\
            $\sigma^2=0.01$ & SAT & SAT & SAT & SAT \\
            $\sigma^2=0.05$ & SAT & SAT & SAT & SAT \\
        \bottomrule
    \end{tabular}
    \caption{Satisfaction of Equation (\ref{eq:specification}) for various obstacle threshold probabilities $\theta$ and noise variances $\sigma^2$.}
    \label{tab:stat_mc}
\end{table}

Recall also that the network's inputs and outputs are characterized by the history of changes in position
$$\{\Delta\mathbf{p}_i + \epsilon_i \}_{i=1}^n, ~~~ \epsilon_i \sim G(0,\sigma^2).$$
We test robustness to perception uncertainty by adjusting the variance, $\sigma^2$, of the noise added to position measurements. The results in Table \ref{tab:stat_mc} for different values of $\sigma$ demonstrate that predictions are not only robust to this sort of additive perception uncertainty, but in fact the addition of noise improves the quality of predictions. Intuitively, this suggests that noisier training data results in higher uncertainty estimates, but at the price of reduced precision. We leave a more thorough characterization of this tradeoff as a topic of possible future work. 

\section{Provable Correctness}\label{sec:proof}

In this section, we show that a known value of Equation (\ref{eqn:model_checking_spec}), which we approximated in the previous section, allows us to derive probabilistic guarantees on the performance of our collision avoidance approach. Specifically, we will show that we can obtain bounds on the probability of collision in multi-agent scenarios as well as for $N$-agent reciprocal collision avoidance. The proofs, omitted here for brevity, are available at {\small\texttt{\url{https://arxiv.org/abs/1811.01075}}}.
    
First, consider a scenario in which agent $A$ uses Algorithm \ref{alg:whole_system} to avoid a collision with arbitrarily moving obstacle $B$. We denote the position of agent $A$ at timestep $n$ by $\mathbf{p}_n^A$.
\begin{proposition}{\textit{Single-Agent Collision Avoidance.}}\label{prop:1_agent}

    Assume that:
    \begin{enumerate}
        \item We can guarantee the accuracy of predictions $e_k$ in the sense that $\mathbb{P}(\mathbf{p}^B_{n+k} \in e^B_k  ~ \forall ~ k \in [1,m]) \geq \theta$.\label{assumption:theta1}
        \item There exists some dynamically feasible velocity $\mathbf{v}_{safe} \notin NPVO_{A \mid B}^{\gamma, m}$.\label{assumption:existence1}
    \end{enumerate}  
     Then the probability of a collision between $A$ and $B$ in $m$ timesteps, $\mathbb{P}(collision)$, is bounded by $1-\theta$. 
\end{proposition}

\begin{proof}
    Consider the set of ellipses generated by Algorithm \ref{alg:ellipse}, $\{e_k\}_{k=1}^m$. Recall that the corresponding NPVO is specified by 
    \begin{multline*}
        NPVO_{A \mid B}^{\gamma,m} = \\
        \{\mathbf{v} \mid \exists k \in [1,m] :: \mathbf{v}k + r \in e_k ~ \forall ~ r , |r| < r_{safe} \}.
    \end{multline*}
    By line \ref{alg:whole_system:vsafe} of Algorithm \ref{alg:whole_system} and assumption \ref{assumption:existence1},
    $$ \mathbf{v}_{safe} \notin NPVO_{A \mid B}^{\gamma, m}$$
    which implies
    \begin{equation}\label{eqn:position}
        \begin{split}
        \mathbf{p}^A_{n+k} = \mathbf{p}^A_{n} + r + \mathbf{v}_{safe}k ~~ \notin ~ e_k \\
        \forall ~ k \in [1,m], |r| < r_{safe}.
        \end{split}
    \end{equation}
    Note that a collision between $A$ and $B$ occurs only if
    $$
    \norm{\mathbf{p}^A_{n+k} - \mathbf{p}^B_{n+k}} < r_{safe},
    $$
    which by Equation (\ref{eqn:position}) occurs only if $\mathbf{p}^B_{n+k} \notin e_k \text{ for some } k \in [1,m]$.
    \begin{align*}
        \begin{split}
        \mathbb{P}(\mathbf{p}^B_{n+k} \notin e_k \text{ for any } k \in [1,m]) \\
        = 1 - \mathbb{P}(\mathbf{p}_{n+k} \in e_k  ~ \forall ~ k \in [1,m]) \\
        \leq 1- \theta
        \end{split}    
    \end{align*}
    by Assumption \ref{assumption:theta1}. Therefore $\mathbb{P}(collision) \leq 1-\theta$.
\end{proof}

Next, consider a scenario in which two agents, $A$ and $B$, both use Algorithm \ref{alg:whole_system} to avoid a collision.

\begin{proposition}{\textit{Dual-Agent Collision Avoidance.}}\label{prop:two_agent}

 Assume that:
    \begin{enumerate}
        \item For each agent, we can guarantee the accuracy of predictions $e_k$ in the sense that $\mathbb{P}(\mathbf{p}_{n+k} \in e_k  ~ \forall ~ k \in [1,m]) \geq \theta$.\label{assumption:theta}
        \item For both of the agents, there exists some dynamically feasible velocity outside the velocity obstacle (i.e. $\exists \mathbf{v}^A_{safe} \notin NPVO_{A \mid B}^{\gamma, m}$ for agent $A$, and $ \exists \mathbf{v}^B_{safe} \notin NPVO_{B \mid A}^{\gamma, m}$ for agent $B$).\label{assumption:existence}
    \end{enumerate}  
     Then the probability of a collision between $A$ and $B$ in $k$ timesteps, $\mathbb{P}(collision)$, is bounded by $(1-\theta)^2$. 
\end{proposition}

\begin{proof}
    Let $\mathcal{A}$ and $\mathcal{B}$ denote the events that agents $A$ and $B$ mis-predict the other's motion. That is, 
    $$
    \mathcal{A} ~~:~~ \mathbf{p}_{n+k}^{B} \notin e_k^B \text{ for some } k \in [1,m]
    $$
    $$
    \mathcal{B} ~~:~~ \mathbf{p}_{n+k}^{A} \notin e_k^A \text{ for some } k \in [1,m].
    $$
    Under assumption \ref{assumption:existence}, the only way a collision will occur is if both agents mispredict the other's motion, so
    $$
    \mathbb{P}(collision) \leq \mathbb{P}(\mathcal{A}\cap\mathcal{B})
    $$
    The predictions of agent are made separately and without knowledge of the other, so the events $\mathcal{A}$ and $\mathcal{B}$ are independent. Following Proposition \ref{prop:1_agent}, $\mathbb{P}(\mathcal{A}) = \mathbb{P}(\mathcal{B}) = 1-\theta$, so
    $$
    \mathbb{P}(collision) \leq \mathbb{P}(\mathcal{A})\mathbb{P}(\mathcal{B}) = (1-\theta)^2
    $$
\end{proof}

Next, consider a scenario in which agent $A$ uses Algorithm \ref{alg:multiagent_system} to avoid collisions with $N$ arbitrarily moving obstacles, which we label $B_1, B_2,\dots, B_N$.
\begin{theorem}{\textit{Multi-Agent Collision Avoidance.}}
\label{theorem:multi_agent}

    Assume that:
    \begin{enumerate}
        \item For each obstacle $B_i$, we can guarantee the accuracy of predictions $e^i_k$ in the sense that $\mathbb{P}(\mathbf{p}^{B_i}_{n+k} \in e^{B_i}_k ~ \forall ~ k \in [1,m]) \geq \theta$.
        \item There exists some dynamically feasibly velocity $\mathbf{v}_{safe} \notin NPVO_{A \mid B_1,B_2,...,B_N}^{\gamma, m}$
    \end{enumerate}
     Then the probability of a collision between $A$ and \textit{any} of the obstacles $B_1,B_2,...B_N$, $\mathbb{P}(collision)$, is bounded by $1-\theta^N$.
\end{theorem}

\begin{proof}
    Let $\mathcal{A}_i$ denote the event
    $$
    \mathbf{p}_{n+k}^{B_i} \notin e_k^{B_i} \text{ for some } k \in [1,m], ~ i \in [1,N].
    $$
    Following Proposition \ref{prop:1_agent}, $\mathbb{P}(\mathcal{A}_i) \leq 1 - \theta$
    and therefore $\mathbb{P}(\mathcal{A}_i^C) \geq \theta$, where $\mathcal{A}_i^C$ denotes the complement of $\mathcal{A}_i$. The probability of a collision between $A$ and any of the obstacles $B_1,B_2,...,B_N$ is bounded as follows:
    $$
    \mathbb{P}(collision) \leq \mathbb{P}(\bigcup_{i=1}^N \mathcal{A}_i) = 1-\mathbb{P}(\bigcap_{i=1}^N \mathcal{A}_i^C).
    $$
    Note that since the predictions $e^{B_i}_k$ are calculated from separate instances in Algorithm \ref{alg:multiagent_system}, the events $\mathcal{A}_i$ (and therefore also $\mathcal{A}_i^C$) are mutually independent. Therefore
    $$
    \mathbb{P}(\bigcap_{i=1}^N \mathcal{A}_i^C) = \prod_{i=1}^N\mathbb{P}(\mathcal{A}_i^C) \geq \theta^N
    $$
    and
    $$
    \mathbb{P}(collision) \leq 1-\theta^N.
    $$
\end{proof}

Finally, consider a scenario in which $N$ agents, $A_1, A_2, \dots, A_N$, each use Algorithm \ref{alg:multiagent_system} to avoid collisions with the other $N-1$ agents.

\begin{theorem}{\textit{Reciprocal Collision Avoidance.}}\label{theorem:safety}

    Assume that:
    \begin{enumerate}
        \item We can guarantee that $\mathbb{P}(\mathbf{p}^{A_j}_{n+k} \in e^{A_j}_k ~ \forall ~ k \in [1,m]) \geq \theta$ for any agent $A_i$ relative to any other agent $A_j$ ($i \neq j$). 
        \item For all agents $A_i$, there exists some velocity $\mathbf{v}^{A_i}_{safe} \notin NPVO_{A_i \mid \{A_j\}_{i \neq j}}^{\gamma, m}$.\label{assumption:existence}
    \end{enumerate} 
    Then the probability that there is \textit{any} collision in the workspace, $\mathbb{P}(collision)$, is bounded by $1-(2\theta - \theta^2)^{\frac{1}{2}(N-1)N}$.
\end{theorem}

\begin{proof}
    Let $\mathcal{A}_{ij}$ denote the event that agents $A_i$ and $A_j$ both mispredict each other's motion. That is,
    $$
        \mathcal{A}_{ij} ~:~ (\mathbf{p}_{n+k}^{A_i} \notin e_k^{A_i}) \cap (\mathbf{p}_{n+k}^{A_j} \notin e_k^{A_j}) \text{ for some } k \in [1,m].
    $$
    Note that under assumption \ref{assumption:existence}, a collision between any two agents will only occur if \textit{both} agents mispredict the other's motion. Let $\mathcal{B}_i$ denote the event that agent $A_i$ mutually mispredicts the motion of any other agent $A_j, j>i$:
    $$
    \mathcal{B}_i = \bigcup_{j=i+1}^N\mathcal{A}_{ij}.
    $$
    Any collision in the workspace will include at least one pairwise collision between two agents, so we can bound the probability of any collision in the workspace by
    $$
    \mathbb{P}(collision) \leq \mathbb{P}(\bigcup_{i=1}^N \mathcal{B}_i).
    $$
    We then proceed as follows:
    $$
    \mathbb{P}(\mathcal{B}_i) = \mathbb{P}(\bigcup_{j=i+1}^N \mathcal{A}_{ij}) = 1 - \mathbb{P}(\bigcap_{j=i+1}^N \mathcal{A}^C_{ij})
    $$
    $$
    = 1 - \prod_{j=i+1}^N \mathbb{P}(\mathcal{A}^C_{ij}) = 1 - (1-(1-\theta)^2)^{N-i}
    $$
    since each $\mathcal{A}_{ij}$ is independent, and following Proposition \ref{prop:two_agent}.
    
    Note that $\{\mathcal{B}_i\}_{i=1}^N$ are mutually independent, since each is composed of the union of different sets of mutually independent events ($\mathcal{A}_{ij}$). We can then write
    $$
    \mathbb{P}(collision) \leq \mathbb{P}(\bigcup_{i=1}^N\mathcal{B}_i) 
    = 1 - \mathbb{P}(\bigcap_{i=1}^N\mathcal{B}_i^C)
    $$
    $$
    = 1 - \prod_{i=1}^N\mathbb{P}(\mathcal{B}_i^C) 
    = 1 - (2\theta - \theta^2)^{\frac{1}{2}(N-1)N}
    $$
\end{proof}

\section{Simulation}\label{sec:simulation}

\subsection{Dynamic Obstacle Motion Prediction}

\begin{figure}
    \centering
    \begin{subfigure}{0.18\textwidth}
        \centering
        \includegraphics[width=\textwidth]{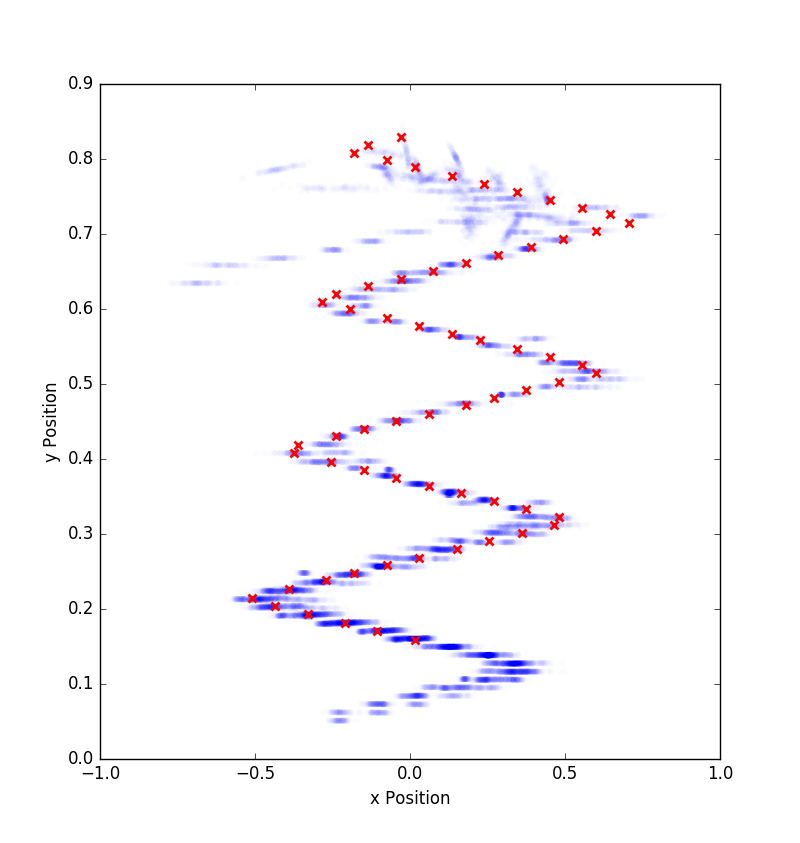}
        \caption{Oscillating}
        \label{fig:online_realtime}
    \end{subfigure}
    \begin{subfigure}{0.25\textwidth}
        \centering
        \includegraphics[width=\textwidth]{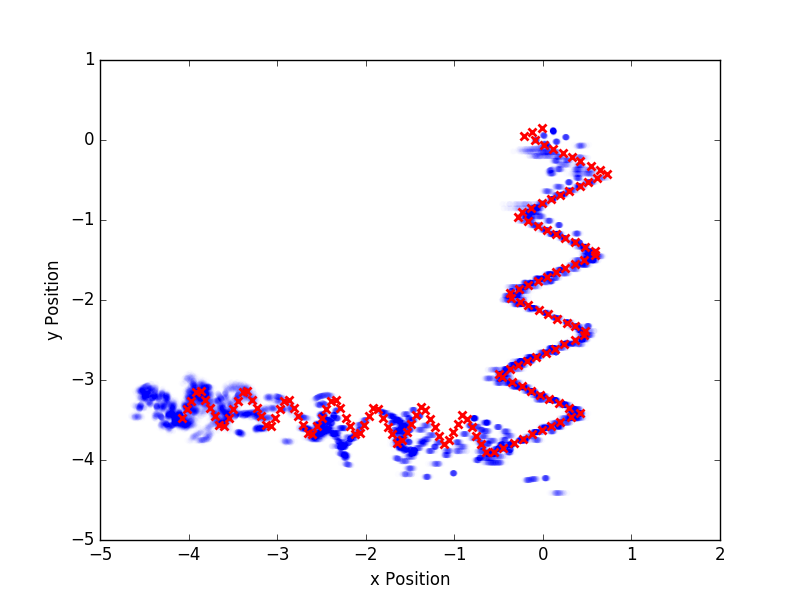}
        \caption{Changing behavior.}
        \label{fig:change_behavior}
    \end{subfigure}
    \begin{subfigure}{0.22\textwidth}
        \centering
        \includegraphics[width=\textwidth]{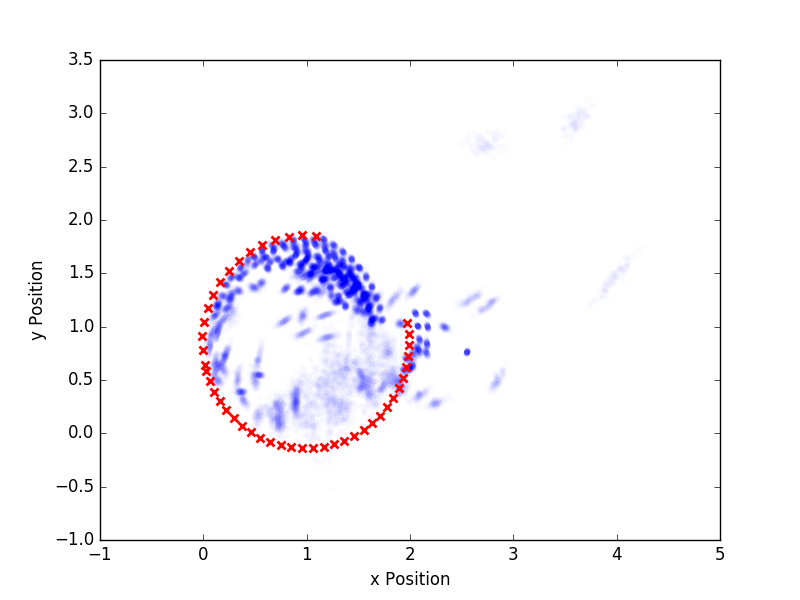}
        \caption{Circle: initial}
        \label{fig:circle_init}
    \end{subfigure}
    \begin{subfigure}{0.22\textwidth}
        \centering
        \includegraphics[width=\textwidth]{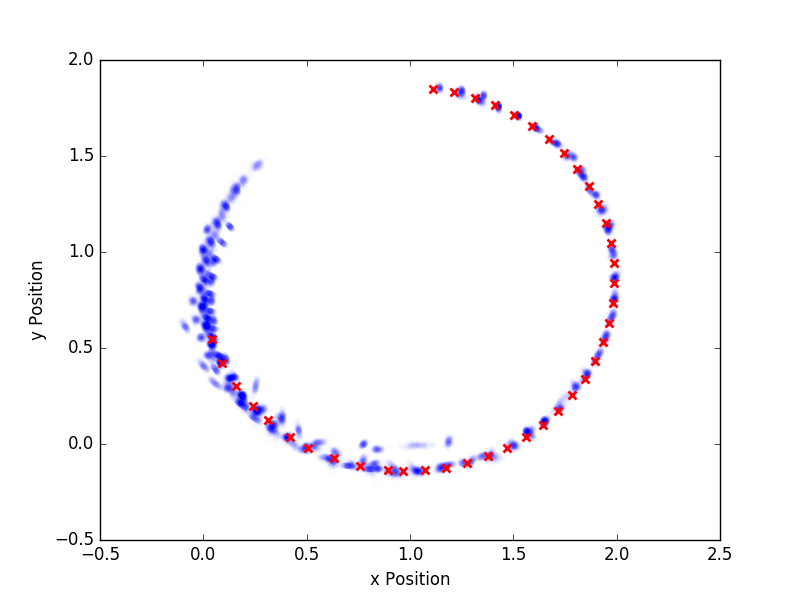}
        \caption{Circle: after $\sim 10$s}
        \label{fig:circle_final}
    \end{subfigure}
    \caption{Predictions of several obstacle motion patterns. Red x's indicate actual obstacle position, blue o's are sample predictions. }
    \label{fig:oscillator_pred}
\end{figure}

We implemented the motion prediction system described in Section \ref{sec:pred} in Python using Tensorflow, ROS, and the Stage simulator \cite{vaughan2008massively}. We use an LSTM network with a hidden layer size of 20, dropout probability $p=0.9$, $N_{iter} = 100$ and a learning rate of 0.003 to predict 10 steps into the future ($m=10$). Note that these hyperparameter values are not optimal, and instead serve as a proof of concept of this methodology. Future work will focus on using Bayesian Optimization \cite{snoek2012practical} to find optimal hyperparameters. The resulting predictions are shown for several patterns of obstacle motion in Figure \ref{fig:oscillator_pred}. 

We compare this system to Gaussian Process (GP) regression using a Matern Kernel \cite{nguyen2009model} and a simple RNN with the same number of hidden layers, dropout probabilities, $N_{iter}$, and learning rate. Our system significantly outperforms the traditional GP approach and offers some improvement over a simple RNN, as shown in Figure \ref{fig:pred_comparison}. 

All predictions were made in real time (2Hz) on a laptop with an Intel i7 processor and 32GB RAM. Code to reproduce these experiments is available at {\small\texttt{\url{https://github.com/vincekurtz/rnn_collvoid}}}.

\begin{figure}
    \centering
    \begin{subfigure}{0.15\textwidth}
        \centering
        \includegraphics[width=\textwidth,height=4cm]{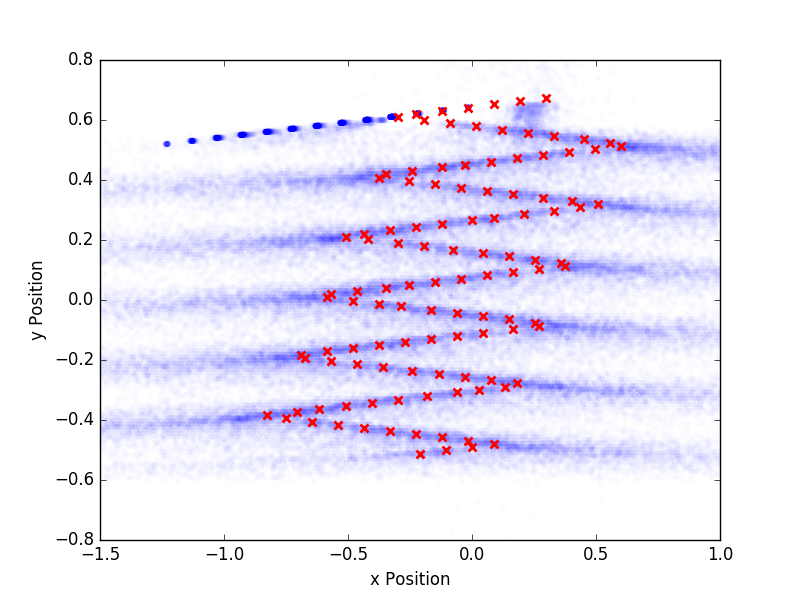}
        \caption{Gaussian Process}
        \label{fig:gp_pred}
    \end{subfigure}
    \begin{subfigure}{0.15\textwidth}
        \centering
        \includegraphics[width=\textwidth,height=4cm]{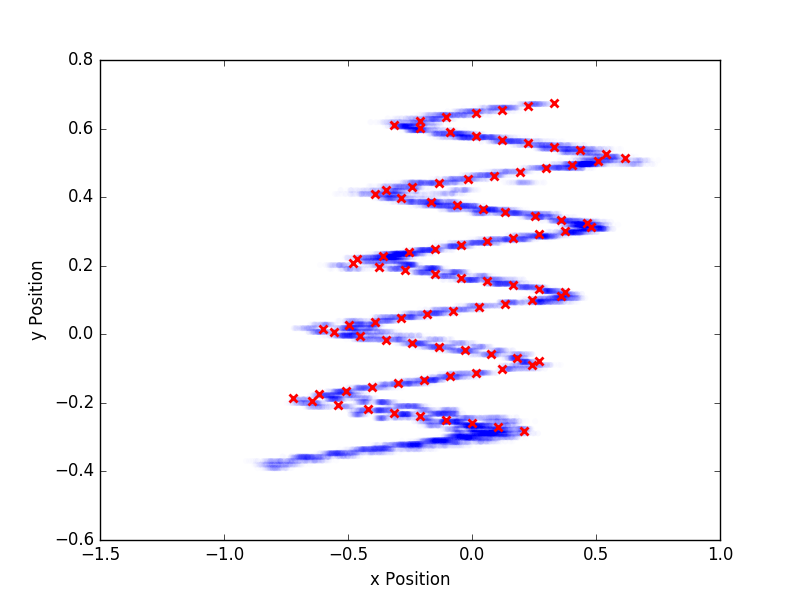}
        \caption{Simple RNN}
        \label{fig:simple_rnn_pred}
    \end{subfigure}
    \begin{subfigure}{0.15\textwidth}
        \centering
        \includegraphics[width=\textwidth,height=4cm]{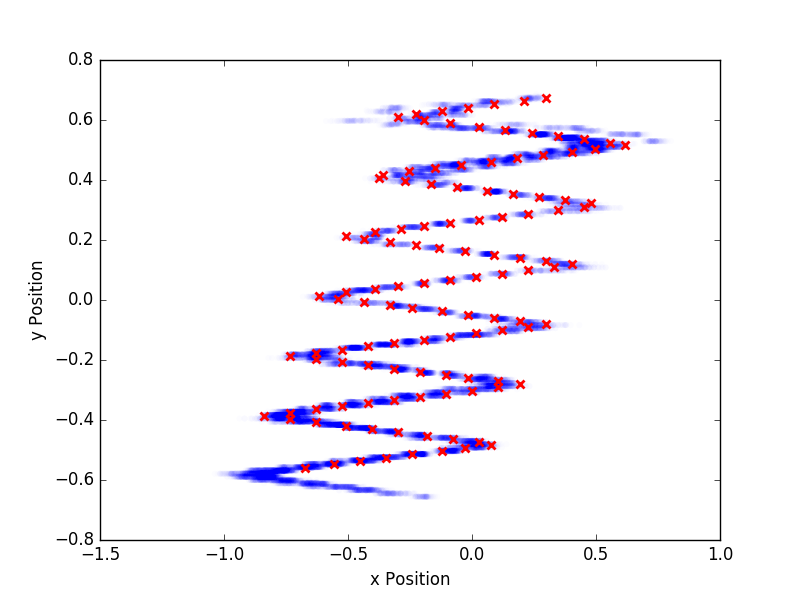}
        \caption{LSTM}
        \label{fig:lstm_pred}
    \end{subfigure}
    \caption{Predicting the motion of a moving obstacle: GP regression fails to accurately predict oscillations, while LSTM demonstrates slightly higher precision than RNN, especially around sharp corners.}
    \label{fig:pred_comparison}
\end{figure}

\subsection{NPVO Controller}

We implemented the NPVO controller described in Section \ref{sec:npvo} in simulation as well, using the \textit{SciPy} \cite{scipy} minimization toolkit to find $\mathbf{v}_{safe} \notin NPVO_{A \mid B}$. Results for several illustrative examples are shown in Figure \ref{fig:controller_success}. Importantly, our implementation was able to avoid a simple obstacle that state-of-the-art methods failed to avoid (Figure \ref{fig:motion}).

\begin{figure}
    \begin{subfigure}{0.40\textwidth}
        \centering
        \includegraphics[width=\textwidth]{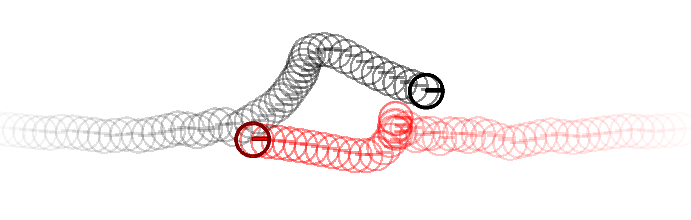}
        \caption{Two agents, $\gamma=0.99$}
        \label{fig:two_agents}
    \end{subfigure}
    \begin{subfigure}{0.21\textwidth}
        \centering
        \includegraphics[width=\textwidth]{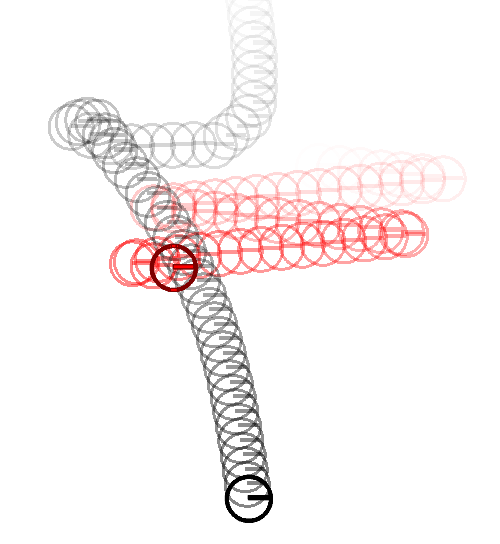}
        \caption{$\gamma=0.5$}
        \label{fig:small_theta}
    \end{subfigure}
    \begin{subfigure}{0.21\textwidth}
        \centering
        \includegraphics[width=\textwidth]{tenstep_theta099.png}
        \caption{$\gamma=0.99$}
        \label{fig:large_theta}
    \end{subfigure}
    \caption{A higher threshold $\gamma$ leads to greater deviation from the desired trajectory, but avoids obstacles by a larger margin.}
    
    \label{fig:controller_success}
\end{figure}

\section{Conclusion}\label{sec:conclusion}

We presented a novel method for predicting the motion of dynamic obstacles from past observations. We introduced the NPVO as a generalization of existing velocity obstacle concepts and demonstrated that an NPVO controller with our prediction system avoids moving obstacles that existing collision avoidance algorithms cannot. We used statistical model checking to (approximately) verify that actual motion remains within certain bounds with a given probability. Finally, we demonstrated the safety and scalability of our approach by proving upper bounds on the probability of collision in multi-agent and reciprocal collision avoidance scenarios.

Future work will focus on characterizing the tradeoff between better uncertainty estimates arsing from additional perception noise and the resulting decreased performance, using Bayesian Optimization techniques to find optimal hyperparameter values, and quantifying the scalability of our approach in scenarios with many obstacles.

\bibliographystyle{plainnat}
\bibliography{references}

\begin{thebibliography}{21}
\providecommand{\natexlab}[1]{#1}
\providecommand{\url}[1]{\texttt{#1}}
\expandafter\ifx\csname urlstyle\endcsname\relax
  \providecommand{\doi}[1]{doi: #1}\else
  \providecommand{\doi}{doi: \begingroup \urlstyle{rm}\Url}\fi

\bibitem[Carlini and Wagner(2017)]{carlini2017towards}
Nicholas Carlini and David Wagner.
\newblock Towards evaluating the robustness of neural networks.
\newblock In \emph{2017 IEEE Symposium on Security and Privacy (SP)}, pages
  39--57. IEEE, 2017.

\bibitem[Ferguson et~al.(2008)Ferguson, Darms, Urmson, and
  Kolski]{ferguson2008detection}
Dave Ferguson, Michael Darms, Chris Urmson, and Sascha Kolski.
\newblock Detection, prediction, and avoidance of dynamic obstacles in urban
  environments.
\newblock In \emph{Intelligent Vehicles Symposium, 2008 IEEE}, pages
  1149--1154. IEEE, 2008.

\bibitem[Fulgenzi et~al.(2007)Fulgenzi, Spalanzani, and
  Laugier]{fulgenzi2007dynamic}
Chiara Fulgenzi, Anne Spalanzani, and Christian Laugier.
\newblock Dynamic obstacle avoidance in uncertain environment combining pvos
  and occupancy grid.
\newblock In \emph{Robotics and Automation, 2007 IEEE International Conference
  on}, pages 1610--1616. IEEE, 2007.

\bibitem[Gal and Ghahramani(2016{\natexlab{a}})]{gal2016dropout}
Yarin Gal and Zoubin Ghahramani.
\newblock Dropout as a bayesian approximation: Representing model uncertainty
  in deep learning.
\newblock In \emph{international conference on machine learning}, pages
  1050--1059, 2016{\natexlab{a}}.

\bibitem[Gal and Ghahramani(2016{\natexlab{b}})]{gal2016theoretically}
Yarin Gal and Zoubin Ghahramani.
\newblock A theoretically grounded application of dropout in recurrent neural
  networks.
\newblock In \emph{Advances in neural information processing systems}, pages
  1019--1027, 2016{\natexlab{b}}.

\bibitem[Hansson and Jonsson(1994)]{hansson1994logic}
Hans Hansson and Bengt Jonsson.
\newblock A logic for reasoning about time and reliability.
\newblock \emph{Formal aspects of computing}, 6\penalty0 (5):\penalty0
  512--535, 1994.

\bibitem[Hochreiter and Schmidhuber(1997)]{hochreiter1997long}
Sepp Hochreiter and J{\"u}rgen Schmidhuber.
\newblock Long short-term memory.
\newblock \emph{Neural computation}, 9\penalty0 (8):\penalty0 1735--1780, 1997.

\bibitem[Huang et~al.(2017)Huang, Kwiatkowska, Wang, and Wu]{huang2017safety}
Xiaowei Huang, Marta Kwiatkowska, Sen Wang, and Min Wu.
\newblock Safety verification of deep neural networks.
\newblock In \emph{International Conference on Computer Aided Verification},
  pages 3--29. Springer, 2017.

\bibitem[Huber et~al.(1964)]{huber1964robust}
Peter~J Huber et~al.
\newblock Robust estimation of a location parameter.
\newblock \emph{The annals of mathematical statistics}, 35\penalty0
  (1):\penalty0 73--101, 1964.

\bibitem[{Hug} et~al.(2018){Hug}, {Becker}, {H{\"u}bner}, and
  {Arens}]{pedestrian2018}
R.~{Hug}, S.~{Becker}, W.~{H{\"u}bner}, and M.~{Arens}.
\newblock {Particle-based pedestrian path prediction using LSTM-MDL models}.
\newblock \emph{ArXiv e-prints}, April 2018.

\bibitem[Jones et~al.(2001--)Jones, Oliphant, Peterson, et~al.]{scipy}
Eric Jones, Travis Oliphant, Pearu Peterson, et~al.
\newblock {SciPy}: Open source scientific tools for {Python}, 2001--.
\newblock URL \url{http://www.scipy.org/}.
\newblock [Online; accessed <today>].

\bibitem[Kahn et~al.(2017)Kahn, Villaflor, Pong, Abbeel, and
  Levine]{kahn2017uncertainty}
Gregory Kahn, Adam Villaflor, Vitchyr Pong, Pieter Abbeel, and Sergey Levine.
\newblock Uncertainty-aware reinforcement learning for collision avoidance.
\newblock \emph{arXiv preprint arXiv:1702.01182}, 2017.

\bibitem[Katz et~al.(2017)Katz, Barrett, Dill, Julian, and
  Kochenderfer]{katz2017reluplex}
Guy Katz, Clark Barrett, David~L Dill, Kyle Julian, and Mykel~J Kochenderfer.
\newblock Reluplex: An efficient smt solver for verifying deep neural networks.
\newblock In \emph{International Conference on Computer Aided Verification},
  pages 97--117. Springer, 2017.

\bibitem[Kingma and Ba(2014)]{kingma2014adam}
Diederik~P Kingma and Jimmy Ba.
\newblock Adam: A method for stochastic optimization.
\newblock \emph{arXiv preprint arXiv:1412.6980}, 2014.

\bibitem[Legay et~al.(2010)Legay, Delahaye, and Bensalem]{legay2010statistical}
Axel Legay, Beno{\^\i}t Delahaye, and Saddek Bensalem.
\newblock Statistical model checking: An overview.
\newblock In \emph{International conference on runtime verification}, pages
  122--135. Springer, 2010.

\bibitem[Nguyen-Tuong et~al.(2009)Nguyen-Tuong, Seeger, and
  Peters]{nguyen2009model}
Duy Nguyen-Tuong, Matthias Seeger, and Jan Peters.
\newblock Model learning with local gaussian process regression.
\newblock \emph{Advanced Robotics}, 23\penalty0 (15):\penalty0 2015--2034,
  2009.

\bibitem[Shiller et~al.(2001)Shiller, Large, and Sekhavat]{shiller2001motion}
Zvi Shiller, Frederic Large, and Sepanta Sekhavat.
\newblock Motion planning in dynamic environments: Obstacles moving along
  arbitrary trajectories.
\newblock In \emph{Robotics and Automation, 2001. Proceedings 2001 ICRA. IEEE
  International Conference on}, volume~4, pages 3716--3721. IEEE, 2001.

\bibitem[Snoek et~al.(2012)Snoek, Larochelle, and Adams]{snoek2012practical}
Jasper Snoek, Hugo Larochelle, and Ryan~P Adams.
\newblock Practical bayesian optimization of machine learning algorithms.
\newblock In \emph{Advances in neural information processing systems}, pages
  2951--2959, 2012.

\bibitem[Van Den~Berg et~al.(2011)Van Den~Berg, Guy, Lin, and
  Manocha]{van2011reciprocal}
Jur Van Den~Berg, Stephen~J Guy, Ming Lin, and Dinesh Manocha.
\newblock Reciprocal n-body collision avoidance.
\newblock In \emph{Robotics research}, pages 3--19. Springer, 2011.

\bibitem[Vaughan(2008)]{vaughan2008massively}
Richard Vaughan.
\newblock Massively multi-robot simulation in stage.
\newblock \emph{Swarm intelligence}, 2\penalty0 (2-4):\penalty0 189--208, 2008.

\bibitem[Wald(1945)]{wald1945sequential}
Abraham Wald.
\newblock Sequential tests of statistical hypotheses.
\newblock \emph{The annals of mathematical statistics}, 16\penalty0
  (2):\penalty0 117--186, 1945.

\end{thebibliography}

\end{document}